\documentclass{article}

\usepackage[T1]{fontenc}    
\usepackage{hyperref}       
\usepackage{url}            
\usepackage{amsmath,amsthm}
\usepackage{booktabs}       
\usepackage{amsfonts}       
\usepackage{nicefrac}       
\usepackage{microtype}      
\usepackage{graphicx}
\usepackage{subfigure}
\usepackage{algorithmicx}
\usepackage{algorithm}
\usepackage{algpseudocode}
\usepackage{multirow}
\newtheorem{theorem}{Theorem}
\newtheorem{definition}{Definition}
\newtheorem{lemma}{Lemma}

\title{On Minimal Accuracy Algorithm Selection in Computer Vision and Intelligent Systems}

\author{
  Martin~Lukac\\
  Department of Computer Science\\
  Nazarbayev University\\
  Astana, 010000, Kazakhstan \\
  \texttt{martin.lukac@nu.edu.kz} \\
  \\
  Kamila~Abdiyeva \\
  Nazarbayev University \\
  Astana, 010000, Kazakkhstan\\
  \texttt{kabdiyeva@nu.edu.kz}\\
  \\
  Michitaka~Kameyama\\
  Ishinomaki Senshu University\\
  Ishinomaki, Japan\\
  \texttt{michikameyama@isenshu-u.ac.jp}\\
}

\begin{document}

\maketitle

\begin{abstract}
In this paper we discuss certain theoretical properties of algorithm selection approach to image processing and to intelligent system in general. We analyze the theoretical limits of algorithm selection with respect to the algorithm selection accuracy. We show the theoretical formulation of a crisp bound on the algorithm selector precision guaranteeing to always obtain better than the best available algorithm result. 
\end{abstract}

\section{Introduction}
Algorithm Selection is a meta-approach that can be seen as an alternative to building more and more complex and general algorithms. Initially introduced by Rice~\cite{rice:76} in the context of task scheduling the algorithm selection has been applied to variety of problems but has never became a mainstream. The main reason is probably the fact that algorithm selection is a meta-approach to problem solving and thus require a relatively large prior knowledge about the problem. However for an efficient algorithm selection, features, attributes, and other types of partial information must be extracted from the input data. The process of obtaining distinctive partial information thus leads to an inevitable computational overhead. Consequently in order to use algorithm selection, the problem must be computationally demanding and must be defined on a large feature space. Such problem space cannot then be efficiently searched with a single algorithm but rather an adaptive selection will provide the correct set of tools to efficiently solve the problem.

Computer vision deals with real world input information: the number of combinations of input features and of environmental conditions is too large for a single algorithm to handle efficiently. Also, computer vision is a very active area and thus a very large number of algorithms already exists and is constantly being developed. Thus applying algorithm selection to computer vision problems is a promising application area.

This paper studies one particular theoretical problem of algorithm selection. We show both theoretically and experimentally that the problem of algorithm selection accuracy is different from the classical stochastic algorithm selection such as the one-armed bandit scheduling~\cite{gittins:79}.  In particular we show that in algorithms where the input information determines the output the algorithm selection accuracy must be at least as accurate as the best algorithm reduced by the variance of the average algorithm output accuracy. 

As a practical verification we apply our results on the semantic segmentation problem. The reason for this choice is that semantic segmentation is a very hard problem and a large amount of algorithms have been and are currently developed. Moreover there is not a single algorithm that outperforms any other one on a case by case basis. This can be observed for instance on the results reported by the evaluation of the VOC2012 data set~\cite{everingham:10}.

This paper is structured as follows. Section~\ref{sec:2} discusses the platform used as a basis for the semantic segmentation algorithm selection problem. Section~\ref{sec:3} describes the problem and the impact of the algorithm selection accuracy. Finally Section~\ref{sec:5} concludes the paper.

\section{Previous Work}
\label{sec:2}

The algorithm selection paradigm was originally introduced by~\cite{rice:76} and since various but only a relatively small amount of applications and studies have been made. Several works proposed general guidelines and studies such as~\cite{tyrell:93,leyton:03,armstrong:06,ali:06,smith:08}, several works considered algorithm selection to the more traditional view on behavior selection in robots~\cite{wang:07,floreano:10} while others considered a more fine grained selection to obtain optimal parameters or best features from a problem space~\cite{yang:97,peng:05,peng:08}. Previous works related to computer vision and image processing includes mainly the work of Yong~\cite{yong:03} that used algorithm selection for segmentation in noisy artificial images and by~\cite{takemoto:09} that used algorithm selection to determine best algorithm for edge detection in biological images. With respect to general robotic processing~\cite{lukac:12b,lukac:13} introduces the concept of algorithm selection into middle and high level processing of natural image segmentation and understanding. More recently~\cite{kim:16} experimented using Bayesian inference for Deep architecture selection in the problem of object recognition with very low improvements. However none of these researches provided a measure of accuracy for the algorithm selection and thus it is impossible to determine how effective the algorithm selection could be.

While in simpler tasks where the source of complexity is well known (artificial noise, contrast) and the input images are limited to a particular category (artificial images, biological cell images) the algorithm selection was successful and the algorithm selection accuracy obtained was very high (95\%)~\cite{yong:03,takemoto:09}. In segmentation of natural images~\cite{lukac:12b} the average algorithm selection accuracy remained under 70\%. In all above described approaches the algorithm selection used for input only local features such as level of noise, color intensity, edges, HOG, wavelets and so on. Moreover none of the algorithms used in the previous works contained any reasoning or manipulation of higher level information related to image content or scene description.

In this work we are focusing on the task of Semantic Segmentation. Semantic segmentation is a task in computer vision that includes the segmentation of an image I into a set of regions $\mathcal{R}_i =\{\mathfrak{r}_{i,0},\ldots \mathfrak{r}_{i,l}\}$ and the labeling of each of the regions with a set of labels $\mathcal{L} = \{\mathfrak{l}_0,\ldots,\mathfrak{l}_k\}$ defined thus by a mapping $S:I\rightarrow \mathcal{R}\times \mathcal{L}$. The process can be described on the pixel level by letting $\mathcal{P}_i=\{\mathfrak{p}_{i,0}\ldots\mathfrak{p}_{i,n}\}$ be the set of pixels constituting the image $I$ and $S:I\rightarrow=\mathcal{P}\times\mathcal{L}$. 

The accuracy of an algorithm performing the semantic segmentation is evaluated by a pixel-wise comparison of a desired ground truth $\mathcal{G}_i=\{\mathfrak{l}_{i,0}\ldots\mathfrak{l}_{i,n}\}$ with the actual output of an algorithm $\mathcal{L}_{\mathfrak{a}}=\{\mathfrak{l}_{0,0}\ldots\mathfrak{l}_{n,j} \}$ using the f-measure~\cite{martin:01}: 
\begin{equation}
	f=\frac{1}{1+\text{result }\text{pixels} - \text{matched} \text{ pixels}}+\frac{\text{ matched }\text{ pixels }}{\text{ground} \text{ truth }\text{ pixels}}
	\label{eq:f}
\end{equation}
where the first term on the right hand side of eq.~\ref{eq:f} represents the precision while the second term represents the reconstruction. 


For our experimental purposes we decided to use a simplified version of the f measure: the reconstruction part of the f measure form eq.~\ref{eq:f}.  The reason for this simplified version of f measure is practical. By using only the precision component of the f measure the average accuracy of semantic segmentation algorithms will be higher which will allows us to determine requirements for efficient algorithm selection with high quality algorithms. 


\subsection{Algorithm Selection Platform}

In this paper we follow the previously introduced framework for high-level information understanding using algorithm selection~\cite{lukac:15}. The algorithm-selection framework is described by the pseudo code~\ref{algo:1}. The platform was introduced with the intent of adding higher level information to the available computer vision algorithms. The goal of the combination improving accuracy selection as well as final performance in algorithms dealing with semantic and symbolic content of the input images. 


\begin{algorithm}[bht]
	\caption{\label{algo:1} Pseudo-code showing the operation of he Autonomous Selection Method (ASM)}
	\begin{algorithmic}[1]
	\State $F_i\gets Features(I_i)$\Comment{Extract Features from Image }$I_i$
	\State $A^0\gets Select(F_i,A)$\Comment{Select the most appropriate algorithm $A_k$ using features $F_i$}
	\State $S_{i}^0\gets A^0(I_i)$\Comment{Process image $I_i$ using algorithm $A_k$}
	\State $G_{i}^0\gets RGraph(S_{ik})$\Comment{Construct multi-relational graph from $S_{ik}$}
	\State $t\gets 0$
	\While{$True$}
		\State $C_{i}^t = Verify(G_i^t)$\Comment{Check $G_i^t$ for semantic contradiction}
		\If{$C_{i}^t == \emptyset$}  break\Comment{If contradiction does not exists}
		\EndIf
		\State $H_i^t = Hypothesis(C_i^t,G_i^t,M)$ \Comment{Generate hypothesis to resolve the contradiction}
		\If{$H_i^t == H_i^{t-1}$} break \Comment{If no new hypothesis exists}
		\EndIf
		\State $F_{C_{i}^t} = Features(I_i,C_{i})$ \Comment{Extract features from the region of contradiction}
		\State $A^t\gets Select(F_{C_{i}^t},H_i^t,A)$\Comment{Select new algorithm using hypothesis and features}
		\If{$A^t == A^{t-1}$} break \Comment{If no new algorithms can be selected}
		\EndIf
		\State $S_{i}^t\gets A^t(I_i)$\Comment{Process the image by the selected algorithm}
		\State $G_{i}^{t+1}\gets Merge(RGraph(S_{i}^t), G_{i}^t)$\Comment{Merge $G_i^t$ graph with the new graph $RGraph(S_i^t)$}
		\State $t\gets t+1$
	\EndWhile
\end{algorithmic}
\end{algorithm}
The ASM starts by extracting features from the whole image (line 1), selects an algorithm (line 2) and processes the input image $I_i$ with the algorithm to obtain a semantic segmentation $S_i^0$ (line 3). A multi-relational graph $G_i^0$ is constructed (line 4) representing inter-object relations. This graph is verified for semantic contradictions (line 7). A contradiction is in this model an relation of size, shape, proximity or occurrence that violates model built from data. 
If contradiction is found a hypothesis that resolves the contradiction (line 10) is proposed. Hypothesis is simply a new object that satisfies more the relational graph. The hypothesis is transformed into a set of attributes, features are extracted from the region that caused the contradiction and both are used to select a new algorithm (line 14). The new algorithm is used to process the image (line 17) and the new graph obtained from resulting semantic segmentation $RGraph(S_i^t)$ is merged with the previous one $G_i^t$. This loop iterates until either no more contradiction exists or no new hypothesis or new algorithm can be selected.

\section{Selection Precision}
\label{sec:3}
One of the main problems of the algorithm selection is the balance of overhead computation and performance achievement. First in order for the algorithm selection to be efficient and effective let's define a cost of computation of resources required for selection. 
\begin{definition}[Computation Cost I]
	is the amount of computation $E(a_j(I_i))$ required to obtain result of processing $a_j(I_i)$ from the initial input $I_j$ using algorithm $a_j\in A$ with $A=\{a_o,\ldots,a_n\}$ being the set of available algorithms.
\end{definition}
\begin{definition}[Algorithm Score]
	is the value obtained by evaluating algorithm's result $\sigma_{ij}$ representing the f value obtained by evaluating the result of $a_j(I_i)$. We will denote the average score of algorithm $a_j$ by $\sigma_j$.
\end{definition}

When discussing the average cost of computation the amount of computation that an algorithm $a_j$ requires to process a data set will be referred to $E(j) = \frac{1}{\vert A\vert}\sum_{i=0}^{\vert A\vert} E(a_j(I_i))$.

Now let the computation be generalized to two subtasks: a selection of algorithm using an algorithm selector method $S$ and the selected algorithm $a_j$.
\begin{definition}[Algorithm Selection]
	is a heuristic function given by the mapping $S: I\rightarrow A$.
\end{definition}
\begin{definition}[Algorithm Selection Cost]
	is the amount of computation $E_S(S(I_i, A))$ required to obtain algorithm $a_j$ -  the best algorithm.
\end{definition}
Thus processing amount required to process $I_i$ using an algorithm selection scheme is given by $E_S+E$.

\begin{definition}[Algorithm Selection Accuracy]
The accuracy of the algorithm selection process $Acc(S)$is evaluated on sample data level. It is given as a percentage representing the amount of data samples for which the selector have chosen the best algorithm $best\_selected$ divided by the total number of data samples $N$ : $Acc(S) = \frac{best\_selected}{N}$
\end{definition}

An accuracy optimal algorithm selection $S$ will select an algorithm for processing the input subject to maximal  f value $f_{max}$: 
\begin{equation}
	S_f: I\xrightarrow{f_{max}} A
\end{equation}


\begin{definition}[Computation Cost II]
	is the amount of computation $E_T(a_j(I_i))$ required to obtain result of processing $a_j(I_i)$ from the initial input $I_i$ using algorithm $a_j$ that was obtained by $a_j = S(I_i, A)$ with $S(\cdot)$ being a selector function minimizing some accuracy function $C$ shown in eq.~\ref{eq:cost}.
\begin{equation}
	C: A\times I\rightarrow\sigma
	\label{eq:cost}
\end{equation}
\end{definition}



\subsection{Binary Case}
\label{sec:bin}

Let there be four algorithms performing semantic segmentation. These algorithms are 1~\cite{ladicky:10}, 2~\cite{carreira:10}, 3~\cite{bharath:14} and 4~\cite{chen:14}. 

To start the analysis we will analyze a completely theoretical and simplified problem case that is however a good start for the algorithm selection accuracy study. In this case we assume that the semantic segmentation is a binary process. Each algorithm score $\sigma$ is either 1 or 0 depending on whether a given algorithm successfully segments an image or not. Such binary results are obtained by taking 100 images from the VOC2012 data set~\cite{everingham:10} and instead of taking the averages of f values of each algorithm the score was simply binarized; the algorithm with highest score of segmentation is given a score 1 all others are given score of 0.

 The scores reported in Table~\ref{tab:segres} are obtained using binary evaluation; each algorithm is evaluated with a binary score. This means that $\sigma_{ij} = 1 \text{ iff } \forall k\neq j,\; \sigma_{ik} \leq \sigma_{ij}$, where $i$ is the index of the input image $I_i$, and $k,j$ are two different algorithms such that $a_k\in \{a_0,\ldots,a_n\}$. Consequently the score of an algorithm $a_j$ is given by $\sigma_{j} = \frac{1}{N} \sum_{i=1}^N \sigma_{ij}$ for all the $N$ images in the data set.
\begin{table}[bht]
	\centering
	\caption{\label{tab:segres} Statistical score of various semantic segmentation algorithms with binary scores}
	\begin{tabular}{|c|c|}
		\hline
		$a_j$&$\sigma_j$\\
		\hline
		1& 21\%\\
		2& 24\%\\
		3& 27\%\\
		4& 28\%\\
		\hline
	\end{tabular}
\end{table}
Notice that because each algorithm is selected only using a binary score the sum of all scores is 100\%. In this case the algorithm selection accuracy is easily approximated because algorith score is binary. The maximum result obtainable is 100\% if accuracy of selection is 100\% (eq.~\ref{eq:maxac}. 
\begin{equation}
	S_{max}(S)\vert f(S(I_i,A)(I_i)) \geq f(a_{l},I_i)
	\label{eq:maxac}
\end{equation}
for $l=0,\ldots,\vert A\vert$  and  $i=0,\ldots, N$.

To determine what is the minimal required accuracy to obtain better score than the best available algorithm observe that statistically such selector $S$ must be accurate at least as many times as the best algorithm is. 

\begin{theorem}[Minimal Accuracy in a Binary Processing Problem]

Let $\sigma_{best}$ indicate the percentage of the $a_{best}$ algorithm, then $acc_{min}(S) = \sigma_{best}$.
\end{theorem}
\begin{proof}
The $a_{best}$ algorithm's $\sigma_j$ value represents how many times it is better than any other available algorithm. This is true for any other pair $a_k, \sigma_k$. Because each of the algorithm is evaluated on a binary scale, $a_j = \sum_{i=0}^{\vert I\vert} S(I_i, A) == a_j$ and $\sum_{j=0}^{\vert A\vert} a_j == a_{best}$ leads directly to $\sigma_{max} = \sigma_j = \sum_{j=0}^{\vert A\vert}\sum_{i=0}^{\vert N\vert} S(I_i, A) == a_j$
\end{proof}

If the $Acc(S) = 100\%$ then the overall score of the algorithm selector based semantic segmentation would result in 100\%  score of semantic segmentation. This is the case only because we are in the binary case where each algorithm is either 100\% correct or 100\% incorrect.

\subsection{Real-Valued Case}
\label{sec:rv}

When the algorithm evaluation is statistical, i.e. each algorithm has score measured by the f measure (or reduced f measure) given in eq.~\ref{eq:f}, the determination of the minimal accuracy ( necessary to always provide a better or at least a result as good as the best algorithm) requires more rigorous analysis. 

If the introduced algorithms were stochastic (random, such that input does not influences output) processes the selection could be studied using the approach used in the scheduling task problem of one or multi armed bandits~\cite{gittins:79,weber:90,gittins:11}. However, algorithms used in semantic segmentation are deterministic (the input determines in most of the cases the output) and have specific output for each input. Consequently purely statistical analysis of their performance is not sufficient and does not allow to determine the minimal required accuracy of the selection mechanism. 

For instance, let the four algorithms from Section~\ref{sec:bin} be used here as well but on a realistic case of image semantic segmentation. Their representative results (reported scores change from the original reported by authors) are shown in Table~\ref{tab:segres1}. Again the results have been obtained only as the average f value of 100 randomly chosen images. This was done in order to remain coherent with the binary case of evaluation in Section~\ref{sec:bin}.

Assume that a set of five images that are processed by each of the available algorithms and the score $\sigma_j$ of each algorithm for each image is shown in Table~\ref{tab:segres2}. Each row shows the f value for each of the images obtained by each of the four algorithms.
\begin{table}[bht]
	\centering
	\caption{Accuracy of four used semantic segmentation algorithms}
	\subtable[][\label{tab:segres1}Statistical accuracy]{\begin{tabular}{|c|c|}
	\hline
		Algorithm&Score $\sigma_j$\\
		\hline
		1& 47\%\\
		2& 48\%\\
		3& 50\%\\
		4& 62\%\\
		\hline
	\end{tabular}}\qquad
	\subtable[][\label{tab:segres2}Exemplar accuracy]{\begin{tabular}{|c||c|c|c|c|}
		\hline
		&\multicolumn{4}{|c|}{Algorithms}\\
		Image ID&1& 2& 3&  4\\
		\hline
		1& 18\%&45\% &\textbf{78}\% &52\% \\
		2& 48\%&65\% &68\% &\textbf{78}\% \\
		3& 50\%&\textbf{70}\% &62\% &53\% \\
		4& \textbf{87}\%&28\% &54\% &44\% \\
		5& 60\%&46\% &35\% &\textbf{76}\% \\
		\hline
		$\sigma_j$&52.6\% &52.6\% &59.4\% &60.6\% \\
		\hline
	\end{tabular}}
\end{table}
Let there be an algorithm selector with statistical precision 80\%. In our case it means that in average it will mismatch one algorithm out of every five choices. Because the algorithms are not stochastic a single mismatch precision is enough to seriously alter the overall result. 

For instance let the selection mechanism be using two algorithms 1 and 4 (second and fifth column in Table~\ref{tab:segres2}). The best selection for the five available images is $\{4, 4, 4, 1, 4\}$ (this is obtained as the maximum of each row between the second and fifth columns). The resulting score is $(52+78+53+87+76)/5 = 69.2\%$ when the accuracy of selection is 100\%. 

Let the $acc(S)= 80\%$ and assume that exactly one of the five algorithms has been chosen wrongly. In this setting let the worst possible score be $(52+78+53+44+76)/5 = 60.8\%$ which is barely higher than the $60.6$\% of the 4$^\text{th}$ algorithm alone. 

Note that other possible selection results will have higher average score of semantic segmentation. Thus, the minimal accuracy of $S$ is strongly depending on the individual performance of each algorithm. This can be seen on the full Table~\ref{tab:segres2}.

Analyzing closer the data from Table~\ref{tab:segres2} various worst cases (with selection accuracy of 80\% and with exactly one out of five choices being wrong) of selection for five algorithms with results introduced in Table~\ref{tab:segres3}.

\begin{table}[bht]
	\centering
	\small
	\caption{\label{tab:segres3} Worst cases of algorithm selection using results from Table~\ref{tab:segres2} with accuracy 80\%}
	\begin{tabular}{|c|c|c|c|c|c|c|}
		\hline
		Image&\multicolumn{6}{|c|}{Cases}\\
		ID&C1& C2& C3& C4& C5&Best\\
		\hline
		1& 18\%&78\% &78\% &78\% &78\%&78\% \\
		2& 78\%&48\% &78\% &78\% &78\%&78\% \\
		3& 70\%&70\% &53\% &70\% &70\%&70\% \\
		4& 87\%&87\% &87\% &28\% &87\%&87\% \\
		5& 76\%&76\% &76\% &76\% &35\%&76\% \\
		\hline
		$\sigma_j$&65.8\% &71.8\% &74.4\% &66\% &69.6\%&77.8\% \\
		\hline
	\end{tabular}
\end{table}

The worst cases presented in Table~\ref{tab:segres3} shows that with a fixed accuracy of the $S$ selector, with discrete amount of wrong selection results and without any knowledge about the relation between the sample data images, the variance of the score of the semantic segmentation can vary greatly. The variance is calculated using the formula  $var=\frac{1}{N}\sum_{i=1}^N \vert\mu-\sigma_i\vert^2 = 11.049$. Here $N=5$ and in general represents the number of statistically sampled results of average algorithm selection scores $\sigma$.

This reasoning and analysis can be expanded further. In particular let now decrease the accuracy of the selector $S$ to 60\%. In this case, let's assume that the accuracy is exact - exactly two out of the five images will be processed by wrong algorithms, the variance will rise to $var=30$.

\begin{figure}[bht]
	\centering
	\subfigure[][\label{fig:vars} Variance of the semantic segmentation score]{\includegraphics[width=0.4\linewidth]{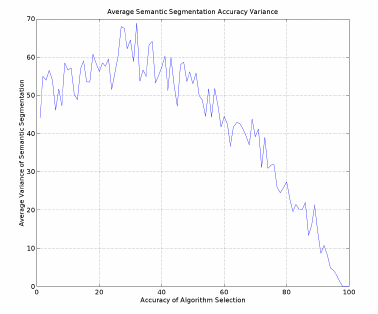}}
	\subfigure[][\label{fig:semseg} Average semantic segmentation score]{\includegraphics[width=0.4\linewidth]{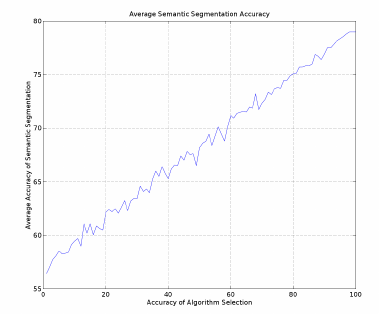}}
	\caption{ Figure showing the variance and the average score of the semantic segmentation as functions of algorithm selection accuracy}
\end{figure}

The variance of the semantic segmentation score as a function of algorithm selection accuracy is shown in Figure~\ref{fig:vars}. The x axis shows the variance of the accuracy of the algorithm selector $S$ and the y axis shows the average variance of the semantic segmentation score. The data used to generate this figure is shown in Table~\ref{tab:segres3}. Each of the point generated is averaged over 255 trials. Notice that as the accuracy of the algorithm selection decreases the variance of the semantic segmentation scores oscillates more. 

Looking closer at the Figure~\ref{fig:vars} it can be observed that in average the variance of the semantic segmentation score is increasing with decreasing algorithm selection accuracy. Observe, that however the variance is highest at around 40\% of algorithm selection accuracy and decreases on both sides. This is to be expected as the expectation is that an algorithm selection with accuracy 0\% will select the worst possible results each time and thus the variance of the average semantic segmentation score will be close to 0.

Thus for the data used here the average selection accuracy of 80\%  will result in a score of 80\% of the best possible segmentation score. Moreover the maximum variance in the semantic segmentation score is 25\% which means that using such selector with 80\% accuracy will be statistically having a lower score $\sigma_{S_{80\%}}$  worse than $\sigma_4$ shown in Table~\ref{tab:segres2}.

Complementary to the variance, the Figure~\ref{fig:semseg} shows the average score of the semantic segmentation as a function of algorithm selection accuracy. 
Notice that unlike the variance that becomes more or less constant under the accuracy of 40\% the average score linearly increases with increasing algorithm selection. Thus one can estimate the minimal accuracy of the algorithm selection by looking at the average score and variance of the semantic segmentation.


It can be then concluded that the accuracy of the algorithm selector can be formulated using the following lemma:
\begin{lemma}[Minimal Algorithm Selection Accuracy]
\label{lem:0}
$acc_{min}(S(I,A))$ using the set of algorithms $A\in\{a_0,\ldots,a_{n-1}\}$ and a set of input images $I=\{i_0,\ldots,i_{m-1}\}$ is the required accuracy such that for any $a_k\in A$, $k=1,\ldots,n-1$ the $\sigma_{k}\leq \sigma_{{S(I,A)}}$ with $\forall k\neq j, \sigma_{k} \geq \sigma_{j}$  and $\sigma_{k}\leq avg(acc(S(I,A)))-avg(var(\sigma_{acc(S(I,A))}))$.
\end{lemma}
Lemma~\ref{lem:0} states that the minimal accuracy $acc_{min(S(I,A))}$ must be such that even the worst case of assignment must be better than the best available algorithm score $\sigma_k$. Such selection mechanism will always result in better result score than any single algorithm would.

For instance, let's look at the example from Table~\ref{tab:segres2}. The maximal possible score of semantic segmentation with $acc(S) = 100\%$ is $\sigma_S= 77.8\%$ (right column in Table~\ref{tab:segres3}). The highest score of any algorithm is $\sigma_4= 60.6\%$ that is $\approx17\%$ below of the $\sigma_S$. This means that using lemma~\ref{lem:0} and Table~\ref{tab:segres3} we have  is $avg(acc(S)) - avg(var(\sigma_{acc(S)}))  \geq 60.6\%$. This can be obtained by looking closer at Figures~\ref{fig:vars} and~\ref{fig:semseg} this can be found $avg(acc{(S)}) - avg(var(\sigma_{acc(S_{85\%})})) \approx 85 - 14 = 61\%$.


%


\section{Experimental Evaluation}

To evaluate the proposed hypothesis about the accuracy of algorithm selection we used the VOC2012 data set and four algorithms introduced in Section~\ref{sec:bin}. For each algorithm we determined the best objects within each image as well as each best image. 
Then, we constructed the set of best possible images by combining best objects within every single image. The comparison of the four used algorithms scores and the selection method using 100\% selection accuracy is shown in Table~\ref{tab:results}.
\begin{table}[bht]
\centering
\caption{\label{tab:results} Comparison of semantic segmentation algorithm score with the algorithm selection assuming a 100\% selection accuracy. Algorithms semantic segmentation score is calculated using the reduced f measure.}
\begin{tabular}{|c|c|c|c|c|c|}
\hline
\multirow{2}{*}{Accuracy Type}&\multicolumn{5}{c|}{Algorithms}\\
\cline{2-6}
&1&2&3&4&Best Selection\\
\hline
Image Accuracy&76.78 &84.03 &85.53 &92.50& 94.7\\
\hline
\end{tabular}
\end{table}
Figure~\ref{fig:voc} shows the variance and accuracy of semantic segmentation using the algorithm selection approach given different levels of selection accuracy. For instance, accuracy of algorithm selection 0\% was calculated by taking the worst result for each input image. 
An accuracy of 30\% was obtained by selecting in 30\% of images the best algorithm while in the remaining 70\% of images select the algorithm with the worst score. 
\begin{figure}[bht]
\centering
\subfigure[][\label{fig:varlast} Average variance of semantic segmentation score ]{\includegraphics[width=0.45\linewidth]{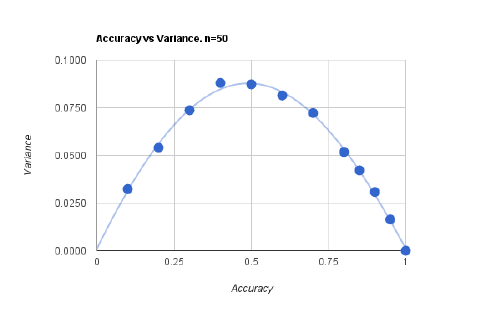}}
\subfigure[][\label{fig:acclast} Average semantic segmentation score]{\includegraphics[width=0.45\linewidth]{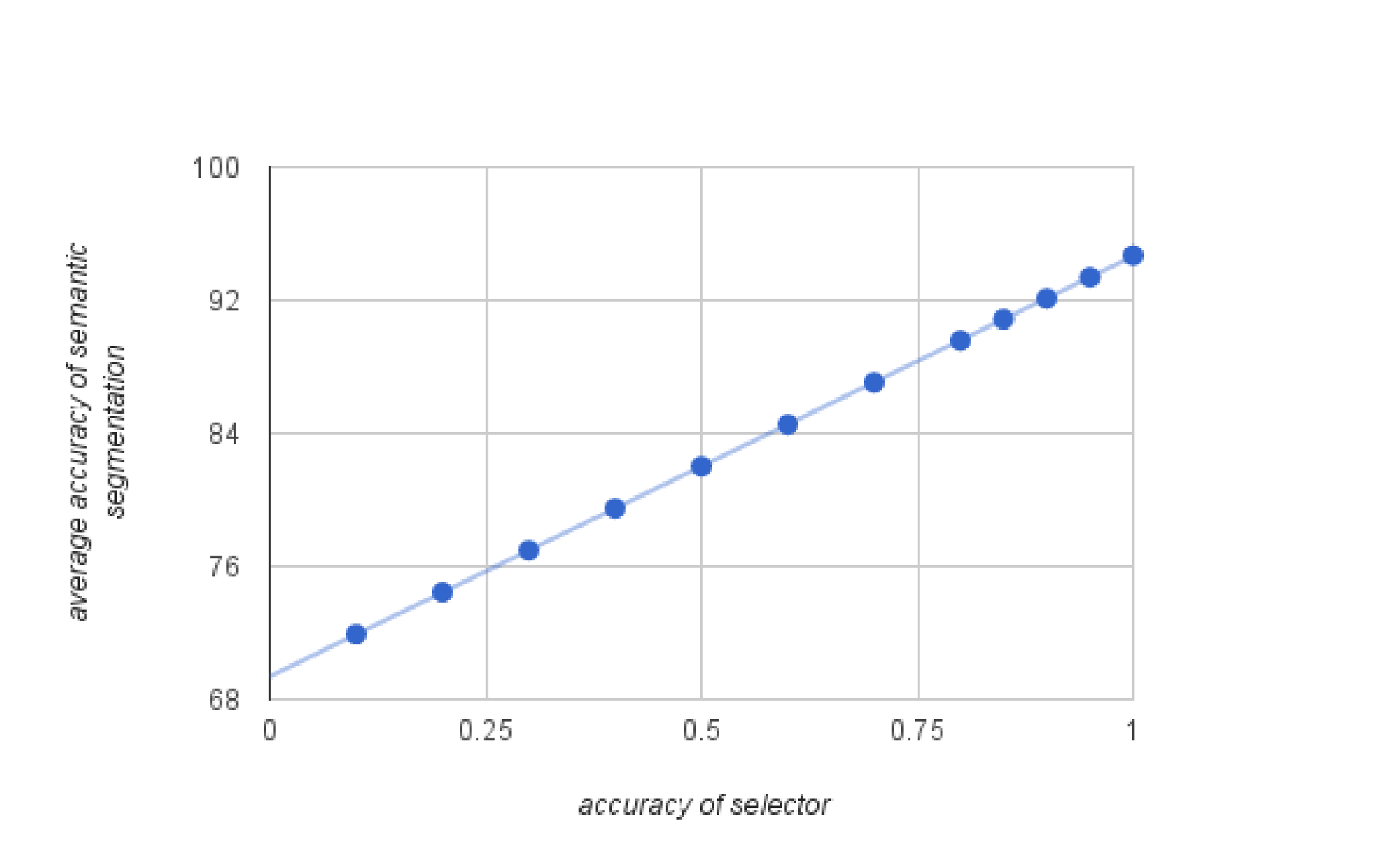}}
\caption{\label{fig:voc} Average variance and accuracy as a function of algorithm selection accuracy calculated using the VOC2012 validation dataset}
\end{figure}
Notice that similarly to the study case in Section~\ref{sec:rv} the trend of the variation is preserved however the averaging over large data set results in a smoother curve and much smaller variation values. Moreover, notice that the largest variation is  at accuracy of selection 50\% because at this accuracy there is the largest variation of the selected algorithms. 

Also observe that our results confirms lemma~\ref{lem:0}. For instance, the highest score of semantic segmentation  shown in Table~\ref{tab:results} is $\sigma_4=92.5\%$ while the highest possible semantic segmentation $\sigma_S= 94.5\%$.
This means that to obtain at least 92.5\% semantic segmentation score the required accuracy must be such that $avg(acc(S)) - avg(var(\sigma_S)) \geq 92.5$.
However $acc(S)-var(\sigma_S) = 92.5-0.023 = 91.477$ is a bit too low and thus we can adjust the accuracy of the selector to 93\% and recalculate the $avg(var(\sigma_{S_{93\%}}))$. The average $\sigma$ is obtained from experimental data. We obtain $acc(S)-var(\sigma_S) = 93-0.0217 = 92.9893$ and thus we have the desired result.

\section{Conclusion}
\label{sec:5}
In this paper we discussed the required precision of algorithm selection method so that one can formulate a robust requirement for performance can be formulated. 

We have shown that for non stochastic algorithms the accuracy of the algorithm selection is directly influenced by the differences between the worst and best cases of each available algorithm. 

An extension of this work is to reformulate the lemma~\ref{lem:0} so that the average of accuracies and of the variance do ot have to be performed but more direct simulation is used.

\bibliographystyle{plain}
\bibliography{../../main}

\end{document}